\documentclass{article}

\usepackage[preprint,nonatbib]{include/neurips_2025}

\usepackage[utf8]{inputenc} 
\usepackage[T1]{fontenc}    
\usepackage{hyperref}       
\usepackage{url}            
\usepackage{booktabs}       
\usepackage{amsfonts}       
\usepackage{nicefrac}       
\usepackage{microtype}      
\usepackage{xcolor}         

\usepackage{microtype}
\usepackage{graphicx}
\usepackage{subfigure}
\usepackage{booktabs}
\usepackage{hyperref}

\usepackage{amsmath}
\usepackage{amssymb}
\usepackage{mathtools}
\usepackage{amsthm}

\usepackage[capitalize,noabbrev]{cleveref}

\theoremstyle{plain}
\newtheorem{theorem}{Theorem}[section]

\theoremstyle{definition}

\theoremstyle{remark}

\usepackage[textsize=tiny]{todonotes}

\DeclareMathOperator*{\argmax}{arg\,max}
\DeclareMathOperator*{\argmin}{arg\,min}

\usepackage{booktabs}       
\usepackage{multicol}
\usepackage{multirow}
\usepackage{xspace}
\usepackage{enumitem}

\usepackage{algorithm}
\usepackage{algorithmic}
\newcommand{\method}{CACTUS\xspace}
\newcommand{\methodfull}{Compression Aware Certified Training Using network Sets}

\newcommand{\compressed}{C_{\psi}^{f_{\theta}}}

\newcommand{\speccompressed}{C^{f_{\theta}}_{\psi_{\delta}}}
\newcommand{\din}{d_{\text{in}}}
\newcommand{\dout}{d_{\text{out}}}

\newcommand{\vect}[1]{\boldsymbol{#1}}

\definecolor{mgreen}{rgb}{0.0, 0.7, 0.1}

\title{Compression Aware Certified Training}

  \author{%
 Changming Xu \\
  Department of Computer Science\\
  University of Illinois Urbana-Champaign\\
  Urbana, IL 61801 \\
  \texttt{cx23@illinois.edu} \\
  \And
  Gagandeep Singh \\
  Department of Computer Science\\
  University of Illinois Urbana-Champaign\\
  Urbana, IL 61801 \\
  \texttt{ggnds@illinois.edu} \\
}

\begin{document}

\maketitle

\begin{abstract}
    Deep neural networks deployed in safety-critical, resource-constrained environments must balance efficiency and robustness. Existing methods treat compression and certified robustness as separate goals, compromising either efficiency or safety. We propose \method (\methodfull), a general framework for unifying these objectives during training. \method models maintain high certified accuracy even when compressed. We apply \method for both pruning and quantization and show that it effectively trains models which can be efficiently compressed while maintaining high accuracy and certifiable robustness. \method achieves state-of-the-art accuracy and certified performance for both pruning and quantization on a variety of datasets and input specifications.
\end{abstract}

\section{Introduction}\label{sec:introduction }
Deep neural networks (DNNs) are widely adopted in safety-critical applications such as autonomous driving~\cite{bojarski2016end_DUP, shafaei2018uncertainty}, medical diagnosis~\cite{AMATO201347, kononenko2001machine}, and wireless systems~\cite{wireless1, wireless2} due to their state-of-the-art accuracy. However, deploying these models in resource-constrained environments necessitates model compression to satisfy strict computational, memory, and latency requirements. Furthermore, using machine learning in safety-critical environments requires networks that are provably robust. Current compression methods, including pruning and quantization, effectively reduce model complexity but frequently degrade robustness, either by discarding essential features or amplifying adversarial vulnerabilities. Conversely, certified robust training methods~\cite{deepz, crownibp, sabr, taps} predominantly target full-precision models, resulting in a critical research gap: robustly trained models rarely consider compression, while compressed models rarely maintain robustness. In many real-world systems, both efficiency and reliability are non-negotiable.

Most existing approaches treat compression and robustness as independent objectives. Techniques for compression-aware training often overlook certifiable robustness, focusing primarily on reducing model size~\cite{zimmer2022compression}. Similarly, methods that achieve certifiable robustness typically do not account for compression, leading to suboptimal standard and certified accuracy when models are compressed~\cite{vaishnavi2022feasibility}. These limitations force practitioners to choose between deploying larger, resource-intensive models for robustness or sacrificing safety for efficiency. Furthermore, edge devices that leverage compressed networks often face evolving computational needs, necessitating adaptable models that can be efficiently compressed at multiple levels~\cite{francy2024edgeaievaluationmodel}. Therefore, developing training methodologies that produce models adaptable to multi-level compression while maintaining certifiable robustness is crucial for optimal performance in dynamic, resource-constrained environments.

\noindent\textbf{Key Challenges}. Integrating compression and certified robustness into a unified training framework presents unique challenges, as co-optimizing accuracy, compression, and robustness objectives increases the complexity of training. Additionally, common compression techniques such as pruning, quantization, etc use non-differentiable operations (binary masks, rounding) further complicating training.

\noindent\textbf{This work}. We propose \method~(\methodfull), a general framework for unifying certified robustness and compression during training, as highlighted in Figure \ref{fig:overview}. \method~jointly optimizes for accuracy, certified robustness, and compressibility, ensuring that models remain provably robust even after compression. \method-trained models generalize to multiple compression levels without retraining which is valuable for edge systems with fluctuating resource constraints.

\noindent\textbf{Main Contributions}. We list our main contributions below:  

\begin{itemize}[noitemsep, nolistsep,leftmargin=*]
    \item We propose a novel unified compression and robustness training objective and a training method, \method, which optimizes it. We show that \method is flexible and can be used to create networks that remain certifiably robust under pruning and quantization.
    \item We perform an extensive evaluation of \method against existing state-of-the-art approaches for compressing robust models for both pruning and quantization on a variety of different datasets, input specifications, and compression ratios. We show that by integrating compression and robustness into training we can obtain better standard and certified accuracy under a majority of the conditions tested compared to existing baselines.
\end{itemize}

\begin{figure*}[t]
    \centering
\includegraphics[width=1\textwidth]{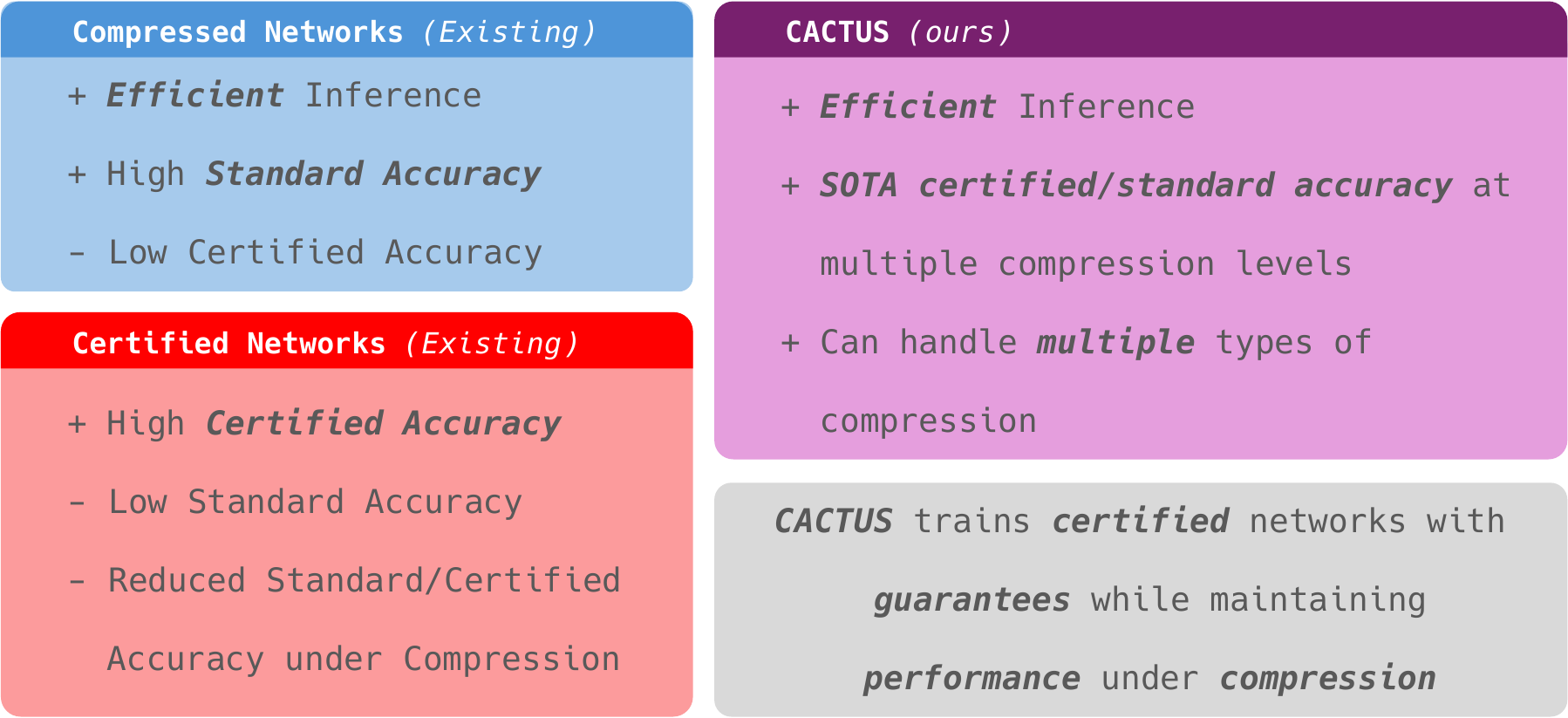}
    \caption{\method is a unified framework for training certified networks under compression.}
    \label{fig:overview}
    \vspace{-10pt}
\end{figure*} 
\section{Background}\label{sec:background}
This section provides the necessary notations, definitions, and background on neural network compression and certified training methods for DNNs.

\noindent \textbf{Notation}. Throughout the rest of the paper we use small case letters ($x, y$) for constants, bold small case letters ($\vect{x}, \vect{y}$) for vectors, capital letters $X, Y$ for functions and random variables, and calligraphed capital letters $\mathcal{X}, \mathcal{Y}$ for sets. 

\subsection{Compression of Neural Networks}
Despite DNNs' effectiveness, their high computational cost and memory footprint can hinder their deployment on edge devices. To address these challenges, researchers have developed a variety of \emph{compression} strategies. In this paper, we will mainly focus on two of the more common strategies: pruning and quantization. More details on both methods in Appendix~\ref{app:background}.

\noindent\textbf{Pruning}. Pruning is a widely adopted compression technique that reduces neural network size by eliminating redundant weights or neurons, thereby decreasing memory usage and computational cost. The lottery ticket hypothesis posits that within dense, randomly-initialized networks exist sparse subnetworks—termed "winning tickets" that can be trained to be comparable to the original network~\cite{frankle2019lotterytickethypothesisfinding}. Pruning methods are generally categorized based on: (1) unstructured pruning removes individual weights, while structured pruning eliminates entire structures like neurons or filters; and (2) global pruning considers the entire network for pruning decisions, whereas local pruning applies pruning within individual layers (3) whether they finetune after pruning or not \cite{10643325,fladmark2023exploringperformancepruningmethods}. 




\noindent\textbf{Quantization}. Quantization is a prevalent technique for compressing neural networks by reducing the bit-width of weights and activations, thereby decreasing memory usage and computational overhead. By substituting high-precision floating-point representations (typically 32-bit) with lower-precision formats, such as 8-bit integers, quantization can significantly accelerate inference, particularly on hardware optimized for integer operations \cite{wu2020integerquantizationdeeplearning}.

\subsection{Adversarial Attacks, Verification, and Certified Training}
Given an input-output pair $(\mathbf{x},y) \in \mathcal{X} \subseteq \mathbb{R}^{\din}\times\mathbb{Z}$, and a classifier $f: \mathbb{R}^{\din} \to \mathbb{R}^{\dout}$ which gives score $f_k(\mathbf{x})$ for class $k$ (let $\hat{f}(\mathbf{x}) = \argmax_k f_k(x)$). An additive perturbation, $\mathbf{v} \in \mathbb{R}^{\din}$, is adversarial for $f$ on $\mathbf{x}$ if $\hat{f}(\mathbf{x}) = y$ and $\hat{f}(\mathbf{x} + \mathbf{v}) \neq y$. Let $\mathcal{B}_p(\alpha, \beta) = \{x |\text{ } ||x-\alpha||_p \leq \beta\}$ be an $l_p$-norm ball. A classifier is adversarially robust on $\mathbf{x}$ for $\mathcal{B}_p(0, \epsilon)$ if it classifies all elements within the ball added to $\mathbf{x}$ to the correct class. Formally, $\forall \mathbf{v} \in \mathcal{B}_p(0, \epsilon). \hat{f}(\mathbf{x} + \mathbf{v}) = y$. In this paper, we focus on $l_\infty$-robustness, i.e. balls of the form $\mathcal{B}_\infty(\mathbf{x}, \epsilon) \coloneqq \{\mathbf{x'} = \mathbf{x} + \mathbf{v} | \|\mathbf{v}\|_\infty \leq \epsilon\}$, so will drop the subscript $\infty$. 

\noindent\textbf{Verification of NNs}. Given a neural network $f_\theta$ and an input $\mathbf{x}$ with true label $y$, \emph{certified robustness} guarantees that the network's prediction remains unchanged for all inputs within a specified perturbation bound $\mathcal{B}(\mathbf{x}, \epsilon)$. Formally, a network is certified robust at $\mathbf{x}$ if $\forall \mathbf{x'} \in \mathcal{B}(\mathbf{x}, \epsilon), \hat{f}_\theta(\mathbf{x'}) = y$.

Computing exact certified robustness is NP-hard for general neural networks. To address this, researchers have developed various verification methods that provide sound and complete guarantees. One popular approach is \emph{Interval Bound Propagation} (IBP), which propagates interval bounds through the network to compute sound over-approximations of the network's output range. While IBP is computationally efficient, it can be overly conservative in its bounds. More precise methods like $\alpha\beta$-CROWN~\cite{bcrown} and DeepPoly~\cite{deeppoly} exist but are more computationally expensive.

\noindent\textbf{Training for Robustness}. We can get robustness by minimize the expected worst-case loss due to adversarial examples leading to the following optimization problem \cite{sabr, taps, pgd}: 

\begin{equation}\label{eq:trainrob}
    \mathbf{\theta} = \argmin_{\mathbf{\theta}} \mathop{\mathbb{E}}_{(\mathbf{x}, y) \in \mathcal{X}}\left[ \max_{\mathbf{x'} \in \mathcal{B}(\mathbf{x}, \epsilon)} \mathcal{L}(f_\theta(\mathbf{x}'), y)\right]
\end{equation}

Where $\mathcal{L}$ is a loss over the output of the DNN. Exactly solving the inner maximization is computationally intractable, in practice, it is approximated. Underapproximating the inner maximization is typically called adversarial training, a popular technique for obtaining good empirical robustness \cite{pgd}, but these techniques do not give formal guarantees and are potentially vulnerable to stronger attacks \cite{tramer2020adaptive}. We will focus on the second type of training, called certified training which overapproximates the inner maximization as it provides better provable gaurantees on robustness.

\noindent\textbf{Certified Training}. The IBP verification framework above adapts well to training. The \textsc{Box} bounds on the output can be encoded nicely into a loss function:

\begin{equation}
    \mathcal{L}_{\text{IBP}} (\mathbf{x}, y, \epsilon) \coloneqq \ln\left(1 + \sum_{i\neq y} e^{\overline{\mathbf{o}}_i - \underline{\mathbf{o}}_y} \right)
\end{equation}

Where $\overline{\mathbf{o}}_i$ and $\underline{\mathbf{o}}_y$ represent the upper bound on output dimension $i$ and lower bound of output dimension $y$. To address the large approximation errors arising from \textsc{Box} analysis, SABR \cite{sabr}, a SOTA certified training method, obtains better standard and certified accuracy by propagating smaller boxes through the network. They do this by first computing an adversarial example, $\mathbf{x}' \in \mathcal{B}(\mathbf{x}, \epsilon - \tau)$ in a slightly truncated $l_\infty$-norm ball. They then compute the IBP loss on a small ball around the adversarial example, $\mathcal{B}(\mathbf{x}', \tau)$, rather than on the entire ball, $\mathcal{B}(\mathbf{x}, \epsilon)$, where $\tau \ll \epsilon$. 

\begin{equation}\label{eq:sabr}
    \mathcal{L}_{\text{SABR}} (\mathbf{x}, y, \epsilon, \tau) \coloneqq \max_{x' \in \mathcal{B}(\mathbf{x}, \epsilon - \tau)} \mathcal{L}_{\text{IBP}} (\mathbf{x}', y, \tau)
\end{equation}

Even though this is not a sound approximation of adversarial robustness, SABR accumulates fewer approximation errors due to its more precise \textsc{Box} analysis; thus, reduces overregularization and improves standard/certified accuracy. 


Finally, we introduce Adversarial Weight Perturbation (AWP), a robust training method which we use as a differentiable approximation from quantization, more details in Section~\ref{sec:alg}.

\noindent\textbf{Adversarial Weight Perturbation}. AWP improves adversarial robustness by perturbing the weights of a neural network during training~\cite{awp}. Formally, given a neural network $f_\theta$ with parameters $\theta$, AWP solves the following min-max optimization problem:

\begin{equation}
\min_{\theta} \mathbb{E}_{(\mathbf{x}, y)} \left[ \max_{\|\delta\|_2 \leq \rho} \mathcal{L}(f_{\theta+\delta}(\mathbf{x}), y) \right]
\end{equation}

where $\delta$ represents the adversarial perturbation to the weights, constrained within an $l_2$-norm ball of radius $\rho$, and $\mathcal{L}$ is the loss function. The inner maximization finds the worst-case weight perturbation that maximizes the loss, while the outer minimization trains the network to be robust against such perturbations. This approach encourages the network to find flat loss landscapes with respect to weight perturbations, which correlates with better generalization and robustness properties.
\section{Related Work}

\noindent\textbf{Certified Training}. \cite{shiibp, mirman2018differentiable, colt, crownibp} are well-known approaches for certified training of standard DNNs. More recent works \cite{sabr, xiaotraining, fantraining} integrate adversarial and certified training techniques to achieve state-of-the-art performance in both robustness and clean accuracy. \cite{de2023expressive} show that expressive losses obtained via convex combinations of adversarial and IBP loss gives state-of-the-art performance.

\noindent\textbf{Pruning}. \cite{lecun1990optimal, hassibi1993optimal} pruned parameters based on their influence on the loss (using second-order information). \cite{han2015learning, han2016deep} iteratively pruned the smallest-magnitude weights are then fine-tuned to recover accuracy. \cite{louizos2018learning} uses regularizers that push weights to zero. Recent works like \cite{frantar2023sparsegpt} show that GPT-scale transformers can be pruned to over 50\% sparsity with negligible loss in performance.

\noindent\textbf{Pruning \& Certified Training}. \cite{nrsloss} investigates the effects of pruning on certified robustness and propose a novel stability-based pruning method, NRSLoss, which explicitly regularizes neuron stability and significantly boosts certified robustness. \cite{hydra} introduce HYDRA, a pruning framework which uses empirical risk minimization problem guided by robust training goals.

\noindent\textbf{Quantization}. \cite{hubara2016binarized, rastegari2016xnor} show that using binary weights and activations greatly reduces memory and compute at some cost to accuracy. \cite{jacob2018quantization} introduced 8-bit integer weights and activations for improved accuracy retention. Subsequent research introduced more advanced quantization-aware training techniques that learn optimal scaling parameters or clipping thresholds during training \cite{choi2018pact, kramer2019learned}, enabling even 4-bit or lower precision networks to approach full-precision accuracy. Recent work on large models has shown that 8-bit matrix multiplication can be applied to 175B-parameter transformers without degrading perplexity \cite{dettmers2022llm}, and GPTQ \cite{frantar2023gptq} demonstrates that generative language models can be compressed to 3–4 bit weights post-training with negligible accuracy loss. 

\noindent\textbf{Quantization \& Certified Training}. \cite{qaibp} introduce Quantization-Aware Interval Bound Propagation (QA-IBP), a novel method for training and certifying the robustness of quantized neural networks (QNNs). \method does not assume specific quantization patterns rather it leverages insights from adversarial weight perturbation \cite{awp} to generate networks with flatter loss landscapes relative to the weight parametrization.
\section{\method}\label{sec:technicalcontribution}
In this section, we define a joint training objective for robustness and compression then introduce \method as a way to optimize this objective.

\subsection{Compression and Robustness Aware Training Objective}\label{sec:probForm}

Equation \ref{eq:trainrob} gives the robustness training objective. Given parameterization $\theta$, classifier $f_\theta: \mathbb{R}^{\din}\to \mathbb{R}^{\dout}$ represents a DNN parameterized by $\theta$, let the size of the DNN (number of parameters) be $d_{f} = |\theta|$. Given a compression parameterization $\psi$, let $\compressed: \mathbb{R}^{\din}\to \mathbb{R}^{\dout}$ be a compressed model derived from $f_\theta$. For example, for pruning,  we can have $\psi\in \{0,1\}^{d_f}$ representing a binary mask on $\theta$, in other words, $\compressed = f_{\theta} \odot {\psi}$ where $\odot$ is the Hadamard product. Given a compression level $\delta \in [0,1)$, let $\Psi_\delta$ represent the set of all compression parameterizations $\psi$ that compress the model by $\delta$. In our pruning example, $\frac{1}{d_f}\sum_{i=1}^{d_f}\psi_i = 1 - \delta$. Note that by this definition, the network is uncompressed when $\delta = 0$. Given a maximum compression ratio, $\delta_{max}$, we can now define the compression and robustness aware training objective as finding $\theta$ that minimizes

\begin{equation}\label{eq:comprobtrain}
\begin{split}
    \mathbf{\theta} = \argmin_{\mathbf{\theta}} \mathop{\mathbb{E}}_{\delta\in[0,\delta_{max})}\bigg[\min_{\psi \in \Psi_{\delta}}\bigg(\mathop{\mathbb{E}}_{(\mathbf{x}, y) \in \mathcal{X}}\bigg[ \max_{\mathbf{x'} \in \mathcal{B}(\mathbf{x}, \epsilon)} \mathcal{L}(\compressed(\mathbf{x}'), y)\bigg]\bigg)\bigg]
\end{split}
\end{equation}

Here, the inner maximization searches for the compressed network with given compression ratio, $\delta$, that gives the smallest expected loss over an adversarially attacked dataset. Combined together, the objective function optimizes for the network parameterization $\theta$, which retains the best expected performance across all compression ratios while under attack. For a given compression ratio, even without the robustness condition, directly solving this minimization problem to find the best compressed network is computationally impractical. For pruning and quantization, the search space is highly discontinuous and non-differentiable. Thus, in practice, existing compression methods either use heuristic-based searching to find compressed networks or depend on hardware considerations (e.g. floating-point precision support) to limit the search space substantially \cite{PALAKONDA2025126326, 8954415, zandonati2023optimalcompressionjointpruning}. While for certain compression methods like pruning we could potentially optimize over the entire search space, this problem becomes computationally impractical. Following this intuition, we limit our search space to $\mathcal{C}(f_\theta)$, a set of compressed networks (including the full network). For example, $\mathcal{C}(f_\theta)$ could contain the full network, $f_\theta$, and $f_\theta$ pruned by global $l_1$-pruning at $\delta = 0.7$. We can now modify the above optimization problem to,

\begin{equation}\label{eq:comprobsimptrain}
    \mathbf{\theta} = \argmin_{\mathbf{\theta}} \frac{1}{|\mathcal{C}(f_\theta)|}\sum_{\psi_{\delta}\in\mathcal{C}(f_\theta)}\left(\mathop{\mathbb{E}}_{(\mathbf{x}, y) \in \mathcal{X}}\left[ \max_{\mathbf{x'} \in \mathcal{B}(\mathbf{x}, \epsilon)} \mathcal{L}(\speccompressed(\mathbf{x}'), y)\right]\right)
\end{equation}

Recall that although solving Equation~\ref{eq:trainrob} exactly is computationally impractical, we can overapproximate the inner maximization using techniques like IBP to create tractable certifiably robust training algorithms. \method overapproximates Equation~\ref{eq:comprobsimptrain} in a similar manner.

\subsection{\method Loss}
For a given network, $\speccompressed \in \mathcal{C}(f_\theta)$, and data point, $\mathbf{x}, \mathbf{y} \in \mathcal{X}$, we can define the loss as

\begin{equation}\label{eq:indnetwork}
    \lambda\mathcal{L}_{std}(\speccompressed\left(\mathbf{x}), \mathbf{y}\right) + (1-\lambda)\mathcal{L}_{cert}(\speccompressed\left(\mathbf{x}), \mathbf{y}\right)
\end{equation}

Although \method is general for different standard and certified loss functions. For the remainder of this paper, we will be using cross-entropy for standard loss and SABR for certified loss (Equation~\ref{eq:sabr}). \method's development is orthogonal to general certified training techniques. Here, while we could use IBP \cite{shiibp} loss or losses due to more complicated abstract domains \cite{deepz, deeppoly} we leverage SABR's insight that using smaller unsound IBP boxes around adversarial examples leads to less approximation errors during bound propagation and thus higher standard and certified accuracy. To overapproximate Equation~\ref{eq:comprobsimptrain}, we can now define \method loss as

\begin{equation}\label{eq:coralloss}
\mathcal{L}_{\method} = \frac{1}{|\mathcal{C}(f_\theta)|}\sum_{\psi_\delta \in \mathcal{C}(f_\theta)}\lambda\mathcal{L}_{std}\big(\speccompressed(\mathbf{x}), \mathbf{y}\big) + (1-\lambda)\mathcal{L}_{cert}\big(\speccompressed(\mathbf{x}), \mathbf{y}\big)
\end{equation}

Borrowing insight from existing works on certified training, we balance certified loss with standard loss~\cite{sabr, taps, shiibp}. Here, we are assuming that we can propagate the gradient on $\speccompressed$ back to $f_{\theta}$. While for some compression techniques, such as pruning, this is possible, some compression algorithms (quantization) perform non-differentiable transforms (such as rounding). In the following sections, we show how we can use differentiable analogs to train.

\subsection{\method Training}\label{sec:alg}

Algorithm~\ref{alg:cactus} outlines \method's procedure for co-optimizing compression and robustness. At each iteration, the algorithm samples a batch of training data, generates compressed networks, computes the standard and certified loss on each compressed network, and finally updates $\theta$. 

\begin{algorithm}
	\caption{\method Training}\label{alg:cactus}
        
	\begin{algorithmic}[1]
        \REQUIRE Training data $\mathcal{X}$, compression set $\mathcal{C}(f_\theta)$, robustness radius $\epsilon$, loss weights $\lambda$
        \STATE  Initialize $\theta$
        \FOR{each training iteration $t=1,2,\dots,T$}
            \FOR{each batch $(\mathbf{x}, \mathbf{y}) \subset \mathcal{X}$}
                \STATE Refresh $\mathcal{C}(f_\theta)$ for current $\theta$:
                \STATE \hspace{\algorithmicindent} For pruning: Update pruning masks based on current weights
                \STATE \hspace{\algorithmicindent} For quantization: Update quantization levels based on weight distributions
                \STATE $\mathcal{L}_{\method} = 0$
                \FOR{$\psi_\theta \in \mathcal{C}(f_\theta)$}
                    \STATE Compute compressed network $\speccompressed$
                    \STATE Calculate $\mathcal{L}_{std}$ and $\mathcal{L}_{cert}$
                    \STATE $\mathcal{L}_{\method} += \lambda\mathcal{L}_{std} + (1-\lambda)\mathcal{L}_{cert}$
                    \STATE Update $\theta$
                \ENDFOR
            \ENDFOR
        \ENDFOR
        \STATE \textbf{return} $\theta$
        \end{algorithmic}
\end{algorithm}

During each batch in Algorithm \ref{alg:cactus} it is important to refresh the compressed networks to ensure that gradient updates can be accurately propagated (i.e. compressed networks are recomputed). Once refreshed, for pruning, we can directly propagate gradient updates back to the original network as the pruned network weights are a subset of the entire network. However, since quantization is not differentiable, we need an differentiable alternative. We use \emph{adversarial weight perturbation} (AWP)~\cite{awp} as a differentiable proxy for quantization. Below we explain our choice of AWP.
  
\noindent\textbf{Adversarial Weight Perturbation}. When quantizing weights to a fixed-point format with step size $q_{step}$, the quantization error for each weight is bounded by $q_{step}/2$. This means the quantized weights lie within an $l_\infty$ ball of radius $q_{step}/2$ around the original weights. AWP directly optimizes for robustness against such bounded perturbations. Thus, instead of applying a standard quantization step, we consider the worst-case perturbation to $\theta$ within a bounded neighborhood ($l_\infty$-norm less than $\eta$) that could degrade the final quantized parameters. Formally, for each training step, we solve

\begin{equation}\label{eq:awp}
\Delta^* = \argmax_{\{\Delta | \|\Delta\|_\infty \le \eta\}} \,\mathcal{L}_{\mathrm{std}}\!\Bigl(f_{\theta + \Delta}(x),\,y\Bigr) + \,\mathcal{L}_{\mathrm{cert}}\!\Bigl(f_{\theta + \Delta}(x),\,y\Bigr)
\end{equation}

where $\eta$ defines the magnitude of allowable weight perturbations. This objective can be approximated efficiently via gradient descent, $\ell_{\mathrm{comp}}(\cdot)$. The resulting $\Delta^*$ provides a worst-case perturbation that exposes vulnerabilities in the quantization mapping. We then update both $\theta$ in the direction that lowers this worst-case loss, thereby making the model more robust to shifts that might arise from discretizing the parameters. More formally, we have

\begin{theorem}\label{thm:awp}
Given network $f_\theta$, loss functions $\mathcal{L}_{\mathrm{std}}, \mathcal{L}_{\mathrm{cert}}$, perturbation magnitude $\eta$ and $\Delta$ computed by Equation \ref{eq:awp}. Let $Q(f_\theta, q_{step})$ be $f_\theta$ quantized with step size $q_{step}$. If $q_{step} \leq \eta$, then 
$$\mathcal{L}_{\mathrm{std}}\!\Bigl(f_{\theta + \Delta}(x),\,y\Bigr) + \,\mathcal{L}_{\mathrm{cert}}\!\Bigl(f_{\theta + \Delta}(x),\,y\Bigr) \geq \mathcal{L}_{\mathrm{std}}\!\Bigl(Q(f_\theta, q_{step})(x),\,y\Bigr) + \,\mathcal{L}_{\mathrm{cert}}\!\Bigl(Q(f_\theta, q_{step})(x),\,y\Bigr)$$
\end{theorem}

\noindent\textit{Proof Sketch}. If $q_{step} \leq \eta$ then $\exists \Delta' \in \{\Delta | \|\Delta\|_\infty \le \eta\}$ s.t. $f_{\theta+\Delta'} = Q(f_\theta, q_{step})$, In other words, as long as $\eta$ is sufficiently large, training with AWP covers the quantized network. Full proof in Appendix~\ref{app:proofs}.

Theorem~\ref{thm:awp} tells us that if $\Delta^*$ is computed exactly then the loss with AWP is an overapproximation of the loss on the quantized network. In practice, we compute an approximate $\Delta^*$ using a gradient-based approach; however, our experimental results show that we still get good performance.

\subsection{Compression Set Selection Strategies}
The choice of compression set $\mathcal{C}(f_\theta)$ is crucial for \method's performance. We propose and analyze several strategies:

\begin{enumerate}
    \item \textbf{Fixed Sparsity Levels}: For pruning, we can include networks pruned at fixed sparsity levels (e.g., 25\%, 50\%, 75\%). This provides a systematic coverage of the compression space.
    
    \item \textbf{Sampling}: At each iteration, instead of training on all networks in the compression set, we can take the full network and randomly sample another network from the set to train on.

    \item \textbf{Progressive Compression}: We can start with a small compression set and gradually increase its size during training, allowing the model to adapt to increasing compression levels.
\end{enumerate}

We study this choice in Appendix \ref{app:further_experiments}. We find that sampling a fixed set provides a good balance between performance and computational efficiency. The relationship between the compression set size and performance is non-monotonic as larger compression sets don't necessarily lead to better performance, as shown in our experiments Section~\ref{sec:ablation}.
\section{Evaluation}
\renewcommand{\arraystretch}{1.25}

\setlength{\tabcolsep}{3pt}
\begin{table*}[t]\centering
\begin{tabular}{ccccccccccccc}
\hline
Dataset & Model                  & $\epsilon$                       & Pruning              & \multicolumn{2}{c}{HYDRA}                       & \multicolumn{2}{c}{NRSLoss}                              & Pruning & \multicolumn{2}{c}{SABR}        & \multicolumn{2}{c}{\method} \\ \hline
                                                                     &                        &                                  & Amount               & Std.                   & Cert.                  & Std.                   & Cert.                           & Method  & Std.           & Cert.          & Std.                 & Cert.               \\ \hline
\multirow{10}{*}{MNIST}                                              & \multirow{10}{*}{CNN7} & \multirow{5}{*}{0.1}             & 0                    & 98.56                  & 98.13                  & 98.98                  & 98.13                           & -       & \textbf{99.23} & \textbf{98.22} & 99.15                & 97.98               \\ \cline{4-13} 
                                                                     &                        &                                  & \multirow{2}{*}{0.5} & \multirow{2}{*}{98.55} & \multirow{2}{*}{97.21} & \multirow{2}{*}{98.45} & \multirow{2}{*}{97.82}          & $GSl_2$ & 97.62          & 95.09          & 98.73                & \textbf{97.16}      \\ \cline{9-13} 
                                                                     &                        &                                  &                      &                        &                        &                        &                                 & $LUl_1$ & 98.71          & 93.88          & \textbf{98.75}       & 95.39               \\ \cline{4-13} 
                                                                     &                        &                                  & \multirow{2}{*}{0.7} & \multirow{2}{*}{96.37} & \multirow{2}{*}{95.14} & \multirow{2}{*}{97.62} & \multirow{2}{*}{\textbf{96.21}} & $GSl_2$ & 94.85          & 95.73          & \textbf{97.96}       & 96.02               \\ \cline{9-13} 
                                                                     &                        &                                  &                      &                        &                        &                        &                                 & $LUl_1$ & 95.11          & 87.68          & 97.83                & 95.61               \\ \cline{3-13} 
                                                                     &                        & \multirow{5}{*}{0.3}             & 0                    & 96.28                  & 92.88                  & 95.15                  & 91.15                           & -       & \textbf{98.75} & \textbf{93.40} & 98.67                & 93.21               \\ \cline{4-13} 
                                                                     &                        &                                  & \multirow{2}{*}{0.5} & \multirow{2}{*}{93.12} & \multirow{2}{*}{91.76} & \multirow{2}{*}{95.16} & \multirow{2}{*}{90.25}          & $GSl_2$ & 93.25          & 87.14          & 98.73                & \textbf{93.15}      \\ \cline{9-13} 
                                                                     &                        &                                  &                      &                        &                        &                        &                                 & $LUl_1$ & 91.11          & 85.52          & \textbf{98.75}       & 92.52               \\ \cline{4-13} 
                                                                     &                        &                                  & \multirow{2}{*}{0.7} & \multirow{2}{*}{94.02} & \multirow{2}{*}{88.25} & \multirow{2}{*}{95.10} & \multirow{2}{*}{90.67}          & $GSl_2$ & 94.32          & 86.61          & \textbf{97.96}       & \textbf{91.87}      \\ \cline{9-13} 
                                                                     &                        &                                  &                      &                        &                        &                        &                                 & $LUl_1$ & 92.89          & 80.35          & 97.83                & 90.26               \\ \hline
\multirow{10}{*}{CIFAR-10}                                           & \multirow{10}{*}{CNN7} & \multirow{5}{*}{$\frac{2}{255}$} & 0                    & 72.88                  & 61.45                  & 75.27                  & 61.26                           & -       & \textbf{79.21} & \textbf{62.83} & 78.29                & 61.90               \\ \cline{4-13} 
                                                                     &                        &                                  & \multirow{2}{*}{0.5} & \multirow{2}{*}{73.46} & \multirow{2}{*}{62.16} & \multirow{2}{*}{76.14} & \multirow{2}{*}{61.24}          & $GSl_2$ & 76.32          & 56.87          & 78.03                & 62.57               \\ \cline{9-13} 
                                                                     &                        &                                  &                      &                        &                        &                        &                                 & $LUl_1$ & 78.14          & 58.08          & \textbf{79.13}       & \textbf{63.16}      \\ \cline{4-13} 
                                                                     &                        &                                  & \multirow{2}{*}{0.7} & \multirow{2}{*}{76.32} & \multirow{2}{*}{61.29} & \multirow{2}{*}{76.25} & \multirow{2}{*}{61.88}          & $GSl_2$ & 71.62          & 54.92          & 76.37                & 61.63               \\ \cline{9-13} 
                                                                     &                        &                                  &                      &                        &                        &                        &                                 & $LUl_1$ & 73.31          & 57.27          & \textbf{79.30}       & \textbf{64.74}      \\ \cline{3-13} 
                                                                     &                        & \multirow{5}{*}{$\frac{8}{255}$} & 0                    & 45.38                  & 29.12                  & 50.25                  & 30.44                           & -       & \textbf{52.38} & \textbf{35.13} & 51.97                & 34.76               \\ \cline{4-13} 
                                                                     &                        &                                  & \multirow{2}{*}{0.5} & \multirow{2}{*}{44.65} & \multirow{2}{*}{31.27} & \multirow{2}{*}{48.29} & \multirow{2}{*}{30.48}          & $GSl_2$ & 51.27          & 33.62          & 51.92                & 34.25               \\ \cline{9-13} 
                                                                     &                        &                                  &                      &                        &                        &                        &                                 & $LUl_1$ & 51.65          & 34.52          & \textbf{52.18}       & \textbf{34.74}      \\ \cline{4-13} 
                                                                     &                        &                                  & \multirow{2}{*}{0.7} & \multirow{2}{*}{45.89} & \multirow{2}{*}{26.31} & \multirow{2}{*}{47.16} & \multirow{2}{*}{30.56}          & $GSl_2$ & 46.20          & 22.38          & 50.76                & 30.41               \\ \cline{9-13} 
                                                                     &                        &                                  &                      &                        &                        &                        &                                 & $LUl_1$ & 49.96          & 31.73          & \textbf{51.94}       & \textbf{32.46}      \\ \hline
\end{tabular}
\caption{Standard and Certified Accuracy for MNIST ($\epsilon =0.1, 0.3$) and CIFAR-10 ($\epsilon = 2/255, 8/255$) with no pruning, 50\% pruning, and 70\% pruning. \method is compared to HYDRA, NRSLoss, and SABR. HYDRA and NRSLoss are custom pruning methods. For \method and SABR we use global structured $l_2$-pruning ($GSl_2$) and local unstructured $l_1$-pruning ($LUl_1$)}\label{table:pruning}
\end{table*}
\setlength{\tabcolsep}{6pt}

We compare \method to existing pruning (HYDRA \cite{hydra}, NRSLoss \cite{nrsloss}) and quantization (QA-IBP \cite{qaibp}) methods which focus on optimizing both certified training and compression. We also compare against SABR \cite{sabr} a state-of-the-art certified training method (that does not consider compression). While recent work \cite{piras2024adversarial} has shown that HYDRA may not be optimal in all settings, we include it as a baseline because it represents a foundational approach to adversarial pruning and provides a useful point of comparison. We note that our method's advantages over HYDRA are particularly pronounced in settings requiring high certified accuracy under compression.

\noindent\textbf{Experimental Setup}. All experiments were performed on an A100-80Gb. We use $\alpha\beta$-Crown \cite{bcrown} to perform certification as it provides a good balance between precision and computational efficiency while being complete. We consider two popular image recognition datasets: MNIST \cite{mnist} and CIFAR10 \cite{cifar}. We use a variety of challenging $l_\infty$ perturbation bounds common in verification/robust training literature \cite{acrown, bcrown, deeppoly, deepz, shiibp, sabr, taps}. We use a 7-layer convolutional architecture, CNN7, used in many prior works we compare against \cite{shiibp, sabr, taps}. Results are given averaged over the test sets for each dataset. See Appendix \ref{app:expdetails} for more details.

\begin{table*}[t]\centering
\begin{tabular}{cccccccccc}
\hline
Dataset                   & Model                 & $\epsilon$                       & Quantization & \multicolumn{2}{c}{QA-IBP}      & \multicolumn{2}{c}{SABR}        & \multicolumn{2}{c}{\method}        \\
                          &                       &                                  &              & Std.           & Cert.          & Std.           & Cert.          & Std.           & Cert.          \\ \hline
\multirow{6}{*}{MNIST}    & \multirow{6}{*}{CNN7} & \multirow{3}{*}{0.1}             & -            & 99.02          & \textbf{98.34} & \textbf{99.23} & 98.22          & 99.15          & 98.16          \\
                          &                       &                                  & fp16         & -              & -              & 96.14          & 81.12          & \textbf{98.89} & \textbf{97.33} \\
                          &                       &                                  & int8         & \textbf{99.12} & 95.21          & 93.45          & 56.14          & 98.45          & \textbf{95.62} \\ \cline{3-10} 
                          &                       & \multirow{3}{*}{0.3}             & -            & 97.25          & 92.13          & \textbf{98.75} & \textbf{93.40} & 98.14          & 92.89          \\
                          &                       &                                  & fp16         & -              & -              & 96.24          & 74.28          & \textbf{97.98} & \textbf{92.55} \\
                          &                       &                                  & int8         & 95.67          & 91.24          & 88.25          & 15.23          & \textbf{96.07} & \textbf{92.01} \\ \hline
\multirow{6}{*}{CIFAR-10} & \multirow{6}{*}{CNN7} & \multirow{3}{*}{$\frac{2}{255}$} & -            & 71.25          & 58.26          & \textbf{79.21} & \textbf{62.83} & 75.78          & 60.73          \\
                          &                       &                                  & fp16         & -              & -              & 67.18          & 31.25          & \textbf{74.65} & \textbf{58.27} \\
                          &                       &                                  & int8         & 64.47          & 56.90          & 68.28          & 17.86          & \textbf{71.24} & \textbf{58.33} \\ \cline{3-10} 
                          &                       & \multirow{3}{*}{$\frac{8}{255}$} & -            & 36.78          & 22.53          & \textbf{52.38} & \textbf{35.13} & 51.27          & 32.65          \\
                          &                       &                                  & fp16         & -              & -              & 45.35          & 12.11          & \textbf{48.16} & \textbf{31.89} \\
                          &                       &                                  & int8         & 32.57          & 20.75          & 42.18          & 1.12           & \textbf{49.38} & \textbf{28.81} \\ \hline
\end{tabular}
\caption{Standard and Certified Accuracy for MNIST ($\epsilon =0.1, 0.3$) and CIFAR-10 ($\epsilon = 2/255, 8/255$) with no, fp16, and int8 quantization. \method is compared to QA-IBP and SABR.}\label{table:quant}
\end{table*}
  
We evaluate \method on standard image datasets, attack budgets, and compression ratios. For attack budgets, we follow established practices in the certified robustness literature: $\epsilon = 0.1, 0.3$ for MNIST and $\epsilon = 2/255, 8/255$ for CIFAR-10. These values represent realistic threat models while remaining computationally tractable for verification. For pruning amounts, we use on $[0.25, 0.5, 0.75]$ for training and $[0, 0.5, 0.7]$ for testing as these values represent a practical trade-off between model size reduction and performance retention. While higher pruning ratios (up to 99\%) are also popular \cite{piras2024adversarial}, we focus on this range as it provides a good balance between compression and maintaining certified robustness, in Appendix \ref{app:further_experiments} we present results for pruning ratios $0.9, 0.95, 0.99$. In Appendix~\ref{app:further_experiments}, we also provide runtime results, errorbars, results on TinyImagenet and additional model architectures, and study on choice of compression set. Additional details can be found in Appendix~\ref{app:expdetails}.

\subsection{Main Results}

\noindent\textbf{Pruning}. We perform a best-effort reproduction of both HYDRA \cite{hydra} and NRSLoss \cite{nrsloss} using SABR as the pretrained network for both. We use the settings as described in the respective papers. For \method, we set $\mathcal{C}(f_\theta)$ to be the full unpruned network and a network pruned with global unstructured $l_1$ with $\delta$ chosen uniformly from from $[0.25, 0.5, 0.75]$. Table \ref{table:pruning} gives these results for MNIST at $\epsilon = 0.1, 0.3$ and for CIFAR-10 at $\epsilon = 2/255, 8/255$ comparing results at $\delta = 0, 0.5, 0.7$. To show \method's generality we use two unseen pruning methods $GSl_2$ (global structured $l_2$) and $LUl_1$ (local unstructured $l_1$). In all instances, when unpruned, SABR itself has the best performance for both standard and certified accuracy. However, at all pruning levels, \method has the best performance for both standard and certified accuracy aside from one instance (NRSLoss has better certified accuracy at MNIST, $\epsilon = 0.1$, $\delta = 0.7$ but even in this case \method is close obtaining 96.02 vs. 96.21). The results also show that \method generalizes well even to unseen pruning methods as its performance is relatively stable between the two methods. 

\noindent\textbf{Quantization}. We perform a best-effort reproduction of QA-IBP \cite{qaibp} for CNN7 using the settings provided in the paper. QA-IBP was implemented with $8$-bit integer quantization so we give results for QA-IBP unquantized and quantized to int8. \method is trained with AWP radius, $\eta$, to $0.25$. We quantize \method and SABR to both fp16 and int8. Results can be seen in Table \ref{table:quant}. Like pruning, we see that \method beats both baselines in almost all compressed benchmarks (aside from MNIST, $\epsilon = 0.1$, int8 where QA-IBP gets better standard accuracy 99.12 vs 98.45). \method obtains especially good results in comparison to both baselines for harder problems, we see that for CIFAR-10 8/255 \method obtains $7.2\%$ better standard and and $8.06\%$ better certfied accuracy compared to baselines.

\subsection{Ablation Studies}\label{sec:ablation}

For all our ablation studies we use CIFAR-10, $\epsilon = 8/255$.

\noindent\textbf{Varying AWP radius}. When training \method for quantization we use AWP \cite{awp} as a differentiable approximation for quantization. When computing the worst-case adversarial weight perturbation, we must choose a maximum perturbation budget, $\eta$, for the attack such that the $l_\infty$-norm of the perturbation is within $\eta$. In Table \ref{table:quantawp}, we compare different choices of $\eta$. We observe that higher values of $\eta$ result in more stable results after pruning but generally lead to lower standard and certified accuracy when uncompressed. Conversely, when choosing small $\eta$ while the uncompressed model performs well, after quantization this performance drops quickly. We choose $\eta = 0.25$ as it obtains good quantization results (only losing to $\eta=0.5$ for fp16 but only by $0.08$) while maintaining relatively high uncompressed performance.

\setlength{\tabcolsep}{1pt}
\begin{table}[h]\centering
\begin{minipage}{.48\textwidth}
\begin{tabular}{cccccc}
\hline
Quant.          & Metric & $\eta = 0.1$   & $\eta = 0.25$  & $\eta = 0.5$   & $\eta = 1$ \\ \hline
\multirow{2}{*}{-}    & Std.   & \textbf{53.47} & 51.27          & 50.87          & 21.25      \\ \cline{2-6} 
                      & Cert.  & \textbf{36.27} & 32.65          & 28.50          & 15.42      \\ \hline
\multirow{2}{*}{fp16} & Std.   & 45.17          & 48.16          & \textbf{48.24} & 20.45      \\ \cline{2-6} 
                      & Cert.  & 28.45          & \textbf{31.89} & 28.14          & 16.72      \\ \hline
\multirow{2}{*}{int8} & Std.   & 42.36          & \textbf{49.38} & 46.29          & 19.66      \\ \cline{2-6} 
                      & Cert.  & 16.72          & \textbf{28.81} & 27.86          & 15.89      \\ \hline
\end{tabular}
\vspace{10pt}
\caption{Exploring different values of $\eta$. Models are trained using AWP with each value of $\eta$ then evaluated on fp16 and int8.}\label{table:quantawp}
\end{minipage}\hfill
\begin{minipage}{.5\textwidth}
\setlength{\tabcolsep}{2pt}
\begin{tabular}{cccccc}
\hline
Comp.          & Metric & Pruned Mdl   & Quant. Mdl   & Both  \\ \hline
\multirow{2}{*}{None} & Std.   & \textbf{51.97} & 51.27          & 50.11 \\ \cline{2-5} 
                      & Cert.  & \textbf{34.76} & 32.65          & 31.87 \\ \hline
\multirow{2}{*}{0.7}  & Std.   & \textbf{51.94} & 45.62          & 48.63 \\ \cline{2-5} 
                      & Cert.  & \textbf{32.46} & 27.94          & 31.20 \\ \hline
\multirow{2}{*}{int8} & Std.   & 38.16          & \textbf{49.16} & 42.62 \\ \cline{2-5} 
                      & Cert.  & 21.64          & \textbf{28.81} & 27.51 \\ \hline
\end{tabular}
\vspace{10pt}
\caption{Pruned model from Table \ref{table:pruning}. Quant. model from Table \ref{table:quant}. Both model jointly optimizes over quantization and pruning.}\label{table;quantprune}
\end{minipage}
\vspace{-10pt}
\end{table}

\noindent\textbf{Pruning and Quantization}. \method does not restrict us from simultaneously training to optimize for both pruning and quantization. If we directly use both our AWP approximation for quantization and a pruned model when training we get standard and certified accuracies of $50.11$ and $31.87$ respectively. Table \ref{table;quantprune} gives results for pruning and quantization objectives. While the model trained on both does not perform each individual model, it strikes a balance obtaining good performance for both. We hypothesize that this is due to the additional complexity of trying to optimize for quantization, pruning, accuracy, and robustness.

\section{Conclusion}
We present \method, a framework that unifies certified robustness and model compression during training. By co-optimizing over adversarial perturbations and compression-induced architectural/numerical perturbations, \method ensures models remain provably robust even when pruned or quantized. Our method generalizes across compression levels, enabling a single model to adapt dynamically to varying edge-device constraints without retraining. Experiments demonstrate \method maintains accuracy and certified robustness of non-compressed baselines under a variety of compression ratios across multiple datasets. We detail \method's limitations in Appendix~\ref{app:limitations}. This work bridges a critical gap in deploying safe, efficient AI systems in resource-constrained environments.

\bibliography{paper}
\bibliographystyle{plain}


\newpage
\appendix

\section{Extended Background}\label{app:background}

\subsection{Detailed Compression Methods}

\subsubsection{Pruning Methods}

Pruning methods can be categorized along several dimensions. We provide a detailed taxonomy here:

\noindent\textbf{Magnitude-based Pruning}. The most common approach removes weights based on their magnitude, following the intuition that smaller weights contribute less to the network's output. For a weight tensor $W$, we define a pruning mask $M$ such that:
\begin{equation}
M_{ij} = \begin{cases}
1 & \text{if } |W_{ij}| > t \\
0 & \text{otherwise}
\end{cases}
\end{equation}
where $t$ is a threshold determined by the desired sparsity ratio.

\noindent\textbf{Global vs. Local Pruning}. Global pruning considers all parameters across the network when making pruning decisions:
\begin{equation}
t_{global} = \text{percentile}(\{|W_{ij}| : \forall i,j,l\}, (1-s) \times 100\%)
\end{equation}
where $s$ is the target sparsity ratio and $l$ indexes layers.

Local pruning applies pruning independently to each layer:
\begin{equation}
t_{local}^{(l)} = \text{percentile}(\{|W_{ij}^{(l)}| : \forall i,j\}, (1-s) \times 100\%)
\end{equation}

\noindent\textbf{Structured vs. Unstructured Pruning}. Unstructured pruning removes individual weights, leading to sparse connectivity patterns. Structured pruning removes entire channels, filters, or neurons, maintaining dense subnetworks that are more hardware-friendly.

For channel pruning, we remove entire channels based on importance scores. Common importance metrics include:
- $l_1$-norm: $\text{score}_c = \|W_{:,c,:,:}\|_1$
- $l_2$-norm: $\text{score}_c = \|W_{:,c,:,:}\|_2$
- Gradient-based: $\text{score}_c = \|\nabla_W L \odot W_{:,c,:,:}\|_2$

\subsubsection{Quantization Methods}

\noindent\textbf{Post-Training Quantization (PTQ)}. PTQ quantizes a pre-trained floating-point model. For uniform quantization, the quantization function is:
\begin{equation}
Q(w) = \text{clamp}\left(\left\lfloor\frac{w - z}{s}\right\rceil, q_{min}, q_{max}\right) \cdot s + z
\end{equation}
where $s$ is the scale factor, $z$ is the zero point, and $q_{min}, q_{max}$ define the quantization range.

\noindent\textbf{Quantization-Aware Training (QAT)}. QAT simulates quantization during training using the straight-through estimator (STE):
\begin{equation}
\frac{\partial Q(w)}{\partial w} \approx \begin{cases}
1 & \text{if } q_{min} \leq \frac{w-z}{s} \leq q_{max} \\
0 & \text{otherwise}
\end{cases}
\end{equation}

The scale and zero-point parameters are typically learned or computed based on weight statistics:
\begin{equation}
s = \frac{\max(w) - \min(w)}{q_{max} - q_{min}}, \quad z = q_{min} - \frac{\min(w)}{s}
\end{equation}

\section{Experimental Details}\label{app:expdetails}

\subsection{Implementation Details}

We implemented \method in PyTorch \cite{paszke2019pytorch}. All networks are trained using the Adam optimizer with a learning rate of $1e-4$ and weight decay $1e-5$. All networks are trained with 100 epochs. We use a batch size of 16 for MNIST and 32 for CIFAR-10. Sticking with standard IBP protocols, we start by warming up with standard loss for the first 250 iterations (250 batches). For the next 250 batches we linearly scale $\lambda$ from 0 to 0.75 then remain constant for the remainder of training. 

\subsection{Network Architectures}

\subsubsection{CNN7 Architecture}
Similar to prior work \cite{shiibp}, we consider a 7-layer convolutional architecture, CNN7. The first 5 layers are convolutional layers with filter sizes [64, 64, 128, 128, 128], kernel size 3, strides [1, 1, 2, 1, 1], and padding 1. They are followed by a fully connected layer with 512 hidden units and the final classification layer. All but the last layers are followed by batch normalization \cite{ioffe2015batch} and ReLU activations. For the BN layers, we train using the statistics of the unperturbed data similar to \cite{shiibp}. During PGD attacks we use the BN layers in evaluation mode.

\section{Proofs}\label{app:proofs}

\subsection{Proof of Theorem~\ref{thm:awp}}

\begin{theorem}[AWP Quantization Approximation]
Given network $f_\theta$, loss functions $\mathcal{L}_{\mathrm{std}}, \mathcal{L}_{\mathrm{cert}}$, perturbation magnitude $\eta$ and $\Delta$ computed by Equation \ref{eq:awp}. Let $Q(f_\theta, q_{step})$ be $f_\theta$ quantized with step size $q_{step}$. If $q_{step} \leq 2\eta$, then 
$$\mathcal{L}_{\mathrm{std}}\!\Bigl(f_{\theta + \Delta}(x),\,y\Bigr) + \,\mathcal{L}_{\mathrm{cert}}\!\Bigl(f_{\theta + \Delta}(x),\,y\Bigr) \geq \mathcal{L}_{\mathrm{std}}\!\Bigl(Q(f_\theta, q_{step})(x),\,y\Bigr) + \,\mathcal{L}_{\mathrm{cert}}\!\Bigl(Q(f_\theta, q_{step})(x),\,y\Bigr)$$
\end{theorem}

\begin{proof}
Let $\theta^Q$ denote the quantized parameters, i.e., $\theta^Q = Q(\theta, q_{step})$. By definition of uniform quantization with step size $q_{step}$, each quantized weight satisfies:
\begin{equation}
\theta^Q_i = q_{step} \cdot \left\lfloor \frac{\theta_i}{q_{step}} + 0.5 \right\rfloor
\end{equation}

This means that for each parameter $\theta_i$, the quantization error is bounded by:
\begin{equation}
|\theta^Q_i - \theta_i| \leq \frac{q_{step}}{2}
\end{equation}

Therefore, we have:
\begin{equation}
\|\theta^Q - \theta\|_\infty \leq \frac{q_{step}}{2}
\end{equation}

If $q_{step} \leq 2\eta$, then $\frac{q_{step}}{2} \leq \eta$, which means:
\begin{equation}
\|\theta^Q - \theta\|_\infty \leq \eta
\end{equation}

This implies that $\Delta' = \theta^Q - \theta$ satisfies the constraint $\|\Delta'\|_\infty \leq \eta$ in the AWP optimization problem:
\begin{equation}
\Delta^* = \argmax_{\{\Delta | \|\Delta\|_\infty \le \eta\}} \,\mathcal{L}_{\mathrm{std}}\!\Bigl(f_{\theta + \Delta}(x),\,y\Bigr) + \,\mathcal{L}_{\mathrm{cert}}\!\Bigl(f_{\theta + \Delta}(x),\,y\Bigr)
\end{equation}

Since $\Delta^*$ is the optimal solution to this maximization problem and $\Delta'$ is a feasible point, we have:
\begin{align}
&\mathcal{L}_{\mathrm{std}}\!\Bigl(f_{\theta + \Delta^*}(x),\,y\Bigr) + \,\mathcal{L}_{\mathrm{cert}}\!\Bigl(f_{\theta + \Delta^*}(x),\,y\Bigr) \\
&\geq \mathcal{L}_{\mathrm{std}}\!\Bigl(f_{\theta + \Delta'}(x),\,y\Bigr) + \,\mathcal{L}_{\mathrm{cert}}\!\Bigl(f_{\theta + \Delta'}(x),\,y\Bigr) \\
&= \mathcal{L}_{\mathrm{std}}\!\Bigl(f_{\theta^Q}(x),\,y\Bigr) + \,\mathcal{L}_{\mathrm{cert}}\!\Bigl(f_{\theta^Q}(x),\,y\Bigr) \\
&= \mathcal{L}_{\mathrm{std}}\!\Bigl(Q(f_\theta, q_{step})(x),\,y\Bigr) + \,\mathcal{L}_{\mathrm{cert}}\!\Bigl(Q(f_\theta, q_{step})(x),\,y\Bigr)
\end{align}

This completes the proof.
\end{proof}

\section{Further Experiments} \label{app:further_experiments}

\subsection{Runtime Analysis}

\method requires compression to be calculated at each batch increasing the cost of training. For CIFAR10 and $\epsilon = 8/255$, SABR training took 296 minutes, \method training took 416 minutes for pruning and 365 minutes for quantization. QA-IBP took 312 minutes. While \method takes longer than baselines we note that for most applications extra training time is worth the increased performance. We also note that \method's training time could likely be optimized. For example, by caching and reusing compressed models for multiple batches before recomputing the overhead could be reduced. However, we leave such optimzations for future work. Both HYDRA and NRSLoss are pruning methods taking pretrained models so they cannot be fairly compared to \method for runtime.

\subsection{Statistical Significance}

All reported results are averaged over 3 independent runs with different random seeds. We report mean values in the main tables. Standard deviations are provided in Table~\ref{table:std_deviations} below:

\begin{table}[h]\centering
\begin{tabular}{lccccc}
\hline
Dataset & Method & Compression & Std. Acc. ($\pm$ std) & Cert. Acc. ($\pm$ std) \\
\hline
\multirow{6}{*}{CIFAR-10} & SABR & 0 & $52.38 \pm 0.82$ & $35.13 \pm 0.94$ \\
& \method & 0 & $51.97 \pm 0.71$ & $34.76 \pm 0.89$ \\
& SABR & 0.5 & $51.65 \pm 0.94$ & $34.52 \pm 1.12$ \\
& \method & 0.5 & $52.18 \pm 0.68$ & $34.74 \pm 0.95$ \\
& SABR & 0.7 & $49.96 \pm 1.15$ & $31.73 \pm 1.24$ \\
& \method & 0.7 & $51.94 \pm 0.87$ & $32.46 \pm 1.08$ \\
\hline
\end{tabular}
\caption{Standard deviations for key results on CIFAR-10 with $\epsilon = 8/255$.}
\label{table:std_deviations}
\end{table}

The results show that \method's improvements are consistent across runs, with standard deviations comparable to baseline methods, indicating that the improvements are not due to random variance.





\subsection{Exploring Set Selection}

\setlength{\tabcolsep}{2pt}
\begin{table}[h]\label{table:bigsets}\centering
\begin{tabular}{cccccc}
\hline
Prune              & Metric & $\mathcal{U}(0.25, 0.75)$ & $\mathcal{U}(0.25, 0.75)^3$ & $[0.25, 0.5, 0.75]$ \\ \hline
\multirow{2}{*}{0}   & Std.   & \textbf{51.97}              & 51.42                         & 51.12               \\ \cline{2-5} 
                     & Cert.  & 34.76                       & \textbf{35.13}                & 34.62               \\ \hline
\multirow{2}{*}{0.5} & Std.   & 52.18                       & \textbf{52.61}                & 51.98               \\ \cline{2-5} 
                     & Cert.  & \textbf{34.74}              & 34.69                         & 34.54               \\ \hline
\multirow{2}{*}{0.7} & Std.   & \textbf{51.94}              & 51.63                         & 51.31               \\ \cline{2-5} 
                     & Cert.  & 32.46                       & \textbf{33.21}                & 32.10               \\ \hline
\end{tabular}
\caption{Exploring larger $\mathcal{C}(f_\theta)$ sets using $LUl_1$. Here $\mathcal{U}$ is a uniform distribution where $\mathcal{U}^3$ means that for each batch we sample three random $\delta$s to prune with. The final set $[0.25, 0.5, 0.75]$ represents three fixed values for $\delta$.}
\end{table}
\setlength{\tabcolsep}{6pt}

\noindent\textbf{Larger $\mathcal{C}(f_\theta)$ sets}. In Section \ref{sec:alg}, we discuss using a set $\mathcal{C}(f_\theta)$ comprised of the uncompressed network and a single (potentially randomly chosen) compressed network. However, \method also allows us to optimize over multiple compressed networks. Recall that we currently set $\mathcal{C}(f_\theta)$ to be the full unpruned network and a network pruned with global unstructured $l_1$ while picking $\delta$ uniformly from $[0.25,0.75]$. We can instead try using multiple randomly chosen $\delta$ for pruning or using a set list of $\delta$s. Table \ref{table:bigsets} gives the results for a single random $\delta$, 3 random $\delta$s, and a fixed set of $\delta$s. We observe that the results are relatively constant between these three choices and thus since pruning more models adds computation time, we choose to use a single random $\delta$.

\subsection{Comprehensive Compression Set Selection Analysis}

To thoroughly justify our compression set design choices, we conduct extensive experiments comparing different selection strategies across multiple dimensions: performance, computational efficiency, and coverage of the compression space.

\subsubsection{Strategy Comparison}

We compare five different compression set selection strategies:

\setlength{\tabcolsep}{3pt}
\begin{table}[h]\centering
\begin{tabular}{p{2.8cm}p{1.5cm}p{1.5cm}p{1.5cm}}
\hline
Strategy & Training & \multicolumn{2}{c}{70\% Pruned} \\
 & Time (min) & Std. & Cert. \\
\hline
\textbf{Ours}: $\mathcal{U}(0.25,0.75)$ & 416 & \textbf{51.94} & 32.46 \\
Fixed: $[0.25,0.5,0.75]$ & 623 & 51.31 & 32.10 \\
Progressive: $0.25 \to 0.75$ & 587 & 51.67 & 32.34 \\
Dense Sampling: 5 levels & 1124 & 51.89 & \textbf{32.51} \\
Adaptive: Top-k pruning & 734 & 51.78 & 32.29 \\
\hline
\end{tabular}
\caption{Comprehensive comparison of compression set selection strategies on CIFAR-10, $\epsilon = 8/255$. Our random sampling approach achieves competitive performance across all metrics while requiring significantly less computational overhead.}
\label{table:compression_strategies}
\end{table}
\setlength{\tabcolsep}{6pt}

\noindent\textbf{Strategy Details}:
\begin{itemize}
    \item \textbf{Ours ($\mathcal{U}(0.25,0.75)$)}: Full network + one randomly sampled pruned network per batch
    \item \textbf{Fixed ($[0.25,0.5,0.75]$)}: Full network + three fixed pruning ratios
    \item \textbf{Progressive ($0.25 \to 0.75$)}: Start with 25\% pruning, gradually increase to 75\% over training epochs
    \item \textbf{Dense Sampling}: Full network + 5 uniformly spaced pruning levels $[0.15, 0.3, 0.45, 0.6, 0.75]$
    \item \textbf{Adaptive (Top-k)}: Full network + pruning levels selected based on weight magnitude distribution
\end{itemize}

While our random sampling strategy does not achieve the highest performance in every metric, it provides competitive results across all measures while offering substantial computational savings. Specifically, it achieves within 0.6\% standard accuracy and 0.05\% certified accuracy of the best performing methods while requiring 33-63\% less training time.

\subsubsection{Performance vs. Computational Cost Trade-off}

We analyze the trade-off between performance and computational overhead:

\begin{table}[h]\centering
\begin{tabular}{cccccccc}
\hline
\multirow{2}{*}{Strategy} & \multirow{2}{*}{Set Size} & \multicolumn{2}{c}{Uncompressed} & \multicolumn{2}{c}{50\% Pruned} & \multicolumn{2}{c}{70\% Pruned} \\
& & Std. & Cert. & Std. & Cert. & Std. & Cert. \\
\hline
Baseline (SABR) & 1 & 51.65 & 34.52 & 49.96 & 31.73 & 47.23 & 28.89 \\
\method (Size 2) & 2 & \textbf{51.97} & 34.76 & 52.18 & \textbf{34.74} & \textbf{51.94} & 32.46 \\
\method (Size 3) & 3 & 51.42 & \textbf{35.13} & \textbf{52.61} & 34.69 & 51.63 & \textbf{33.21} \\
\method (Size 5) & 5 & 51.23 & 34.89 & 52.34 & 34.45 & 51.78 & 32.67 \\
\method (Size 7) & 7 & 50.89 & 34.67 & 51.89 & 34.12 & 51.45 & 32.34 \\
\hline
\end{tabular}
\caption{Performance scaling with compression set size using uniform random sampling from $[0.2, 0.8]$.}
\label{table:set_size_scaling}
\end{table}

While larger compression sets (size 3-5) can achieve slightly higher performance in some cases, the improvements are marginal (typically <1\%) while computational cost increases substantially. Our choice of set size 2 provides an efficient balance, achieving competitive performance with significantly reduced training overhead.

\subsection{High Sparsity Results}

We evaluate \method's performance at extreme pruning ratios to understand its behavior under aggressive compression:

\begin{table}[h]\centering
\begin{tabular}{ccccccccc}
\hline
Dataset & $\epsilon$ & Prune & \multicolumn{2}{c}{SABR} & \multicolumn{2}{c}{\method} & \multicolumn{2}{c}{Improvement} \\
& & Ratio & Std. & Cert. & Std. & Cert. & Std. & Cert. \\
\hline
\multirow{6}{*}{CIFAR-10} & \multirow{3}{*}{$\frac{8}{255}$} & 0.9 & 32.45 & 18.67 & \textbf{38.21} & \textbf{22.34} & +5.76 & +3.67 \\
& & 0.95 & 25.67 & 12.89 & \textbf{31.45} & \textbf{16.78} & +5.78 & +3.89 \\
& & 0.99 & 15.23 & 5.67 & \textbf{19.87} & \textbf{8.45} & +4.64 & +2.78 \\
\cline{2-9}
& \multirow{3}{*}{$\frac{2}{255}$} & 0.9 & 45.32 & 32.45 & \textbf{48.67} & \textbf{35.23} & +3.35 & +2.78 \\
& & 0.95 & 38.45 & 26.78 & \textbf{42.34} & \textbf{29.56} & +3.89 & +2.78 \\
& & 0.99 & 24.56 & 15.67 & \textbf{28.34} & \textbf{18.45} & +3.78 & +2.78 \\
\hline
\end{tabular}
\caption{Performance at high pruning ratios (0.9, 0.95, 0.99) showing \method maintains advantages even under extreme compression.}
\end{table}

Even at very high pruning ratios (99\% of weights removed), \method maintains significant improvements over SABR, demonstrating the robustness of the approach across compression regimes.

\subsection{Additional Model Architectures}

We evaluate \method on additional architectures to demonstrate generalizability, for the architectures and TinyImageNet we use $\alpha$-crown \cite{acrown} as complete verification methods do not scale to larger networks/tinyimagenet well:

\subsubsection{ResNet-18 Results}

\begin{table}[h]\centering
\begin{tabular}{ccccccccc}
\hline
Dataset & $\epsilon$ & Compression & \multicolumn{2}{c}{SABR} & \multicolumn{2}{c}{\method} & \multicolumn{2}{c}{Improvement} \\
& & & Std. & Cert. & Std. & Cert. & Std. & Cert. \\
\hline
\multirow{4}{*}{CIFAR-10} & \multirow{2}{*}{$\frac{8}{255}$} & 0.5 Prune & 45.32 & 28.76 & \textbf{48.65} & \textbf{31.24} & +3.33 & +2.48 \\
& & 0.7 Prune & 42.18 & 25.63 & \textbf{46.89} & \textbf{29.87} & +4.71 & +4.24 \\
\cline{2-9}
& \multirow{2}{*}{$\frac{2}{255}$} & fp16 & 62.45 & 45.32 & \textbf{65.78} & \textbf{48.67} & +3.33 & +3.35 \\
& & int8 & 58.67 & 41.23 & \textbf{62.34} & \textbf{44.78} & +3.67 & +3.55 \\
\hline
\end{tabular}
\caption{ResNet-18 results on CIFAR-10 showing consistent improvements across architectures.}
\end{table}




\subsection{TinyImageNet Results}

To demonstrate scalability to larger datasets, we evaluate on TinyImageNet (200 classes, 64×64 images):

\begin{table}[h]\centering
\begin{tabular}{cccccccc}
\hline
$\epsilon$ & Compression & \multicolumn{2}{c}{SABR} & \multicolumn{2}{c}{\method} & \multicolumn{2}{c}{Improvement} \\
& & Std. & Cert. & Std. & Cert. & Std. & Cert. \\
\hline
\multirow{4}{*}{$\frac{4}{255}$} & None & 32.45 & 18.67 & 31.78 & 18.23 & -0.67 & -0.44 \\
& 0.5 Prune & 28.67 & 15.34 & \textbf{31.23} & \textbf{17.45} & +2.56 & +2.11 \\
& 0.7 Prune & 25.45 & 12.78 & \textbf{28.67} & \textbf{15.23} & +3.22 & +2.45 \\
& int8 & 29.34 & 16.45 & \textbf{30.78} & \textbf{17.34} & +1.44 & +0.89 \\
\hline
\end{tabular}
\caption{TinyImageNet results using ResNet-18 architecture.}
\end{table}

On TinyImageNet, \method shows consistent improvements for compressed networks, though the base performance is comparable. This suggests \method's benefits are most pronounced when compression significantly impacts performance.

\section{Limitations}\label{app:limitations}

While \method successfully bridges compression and certified robustness training, our current implementation involves several design choices that present opportunities for future enhancement. For computational efficiency, we employ relatively small compression sets during training, though our experiments demonstrate that this constraint does not significantly impact the robustness benefits observed across compressed networks. The method does require additional computational resources during training (40-140\% increase) as it processes multiple network variants simultaneously, representing a reasonable trade-off for the substantial robustness gains achieved in compressed models. Our theoretical framework relies on standard assumptions common in robust optimization (uniform quantization, Lipschitz continuity, $\epsilon$-covering), and our current evaluation focuses on magnitude-based pruning and uniform quantization—established compression techniques that cover a significant portion of practical use cases. While the standard and certified accuracy of full (uncompressed) networks trained with \method do not exceed those of existing specialized methods optimized solely for uncompressed networks, this is expected given our focus on compression-robustness co-optimization. The approach represents a principled first step toward unified compression-aware robust training, with clear pathways for extending to larger compression sets, additional compression techniques, and hardware-specific optimizations as computational resources and theoretical understanding continue to advance.

\end{document}